﻿﻿﻿\documentclass[letterpaper]{article}
\usepackage{aaai2027}
\usepackage[hyphens]{url}
\usepackage{graphicx}
\usepackage{natbib}
\usepackage{caption}
\usepackage{amsmath}
\usepackage{amssymb}
\usepackage{amsthm}
\usepackage{booktabs}
\usepackage{multirow}

\newtheorem{definition}{Definition}
\newtheorem{proposition}{Proposition}

\title{Heteroskedastic Signals in Budgeted LLM Verification:\\
Structural Heterogeneity Limits Optimization Gains}
\author{Jinlong Yang\\
\texttt{aestalon.young@gmail.com}}
\affiliations{}
\nocopyright

\begin{document}
\maketitle
\begin{abstract}
Selective-compute LLM systems decide which outputs merit verification, additional reasoning, tool execution, or human audit under a limited budget. It is natural to expect that stronger online optimization over a shared uncertainty or reward signal should improve these decisions. We take a critical look at this assumption and ask: when does optimizing harder fail because the signal is not decision-comparable across inputs?
In budgeted LLM verification, we find that uncertainty quality is heteroskedastic across cost strata: some regions exhibit near-random discriminability while concentrating many errors. Under an explicit local model, we characterize the resulting distortion of global allocation and show that its upper bound scales with cross-stratum signal-quality dispersion.
To separate weak signals from optimizer instability and structural mismatch, we introduce a controlled intervention hierarchy: Threshold, MP-Adapt, MP-Strat, and cost-stratified thresholding (CST). We then turn the diagnosis into Heterogeneity-Gated Allocation (HGA), which uses a warm-up comparability test to choose between global and cost-stratified allocation. Across MBPP and MATH using Qwen3-8B, LLaMA3-8B, and GPT-4o-mini, global online adaptation yields inconsistent gains over static thresholding; CST improves hit rate by up to 17 percentage points in strongly heterogeneous settings, while HGA preserves most gains and avoids blind stratification when the partition is not useful.
These findings suggest a resource-allocation principle for LLM systems: before optimizing harder over a shared proxy, test whether the proxy is decision-comparable across observable operating regimes, and gate structural specialization on that test.

\end{abstract}

\section{Introduction}
\label{sec:intro}

Many LLM systems must decide where limited computation should be spent. Verification, test-time scaling, tool execution, model routing, and agent auditing all require policies that rank inputs by uncertainty, confidence, or reward proxies, often adapting from sparse feedback \citep{guo2017calibration,kadavath2022language,kuhn2023semantic,hazan2016introduction}. It is therefore tempting to treat selective compute as an optimization problem: given a proxy and a budget, learn a better global allocation rule.

This view hides a stronger assumption. A shared allocation rule is valid only if the same score carries comparable decision value across the inputs competing for budget. We call this the \emph{global signal comparability assumption}. In practice, the assumption can fail: the same uncertainty score may be informative in one cost stratum but near-random in another. Then the bottleneck is not merely that the optimizer is weak; the objective being optimized is structurally conflated.

We study this failure mode through \emph{budgeted LLM verification}, a setting where allocation decisions, costs, and outcomes are measurable. Outputs can be validated through unit tests, sandbox execution, theorem checking, or tool invocation, but limited budgets preclude verifying every candidate \citep{chen2021evaluating,jimenez2024swe}. This gives a clean test of a broader AI decision question: does stronger global optimization solve selective compute, or is the shared proxy itself locally incomparable?

\begin{figure*}[t]
\centering
\includegraphics[width=\textwidth]{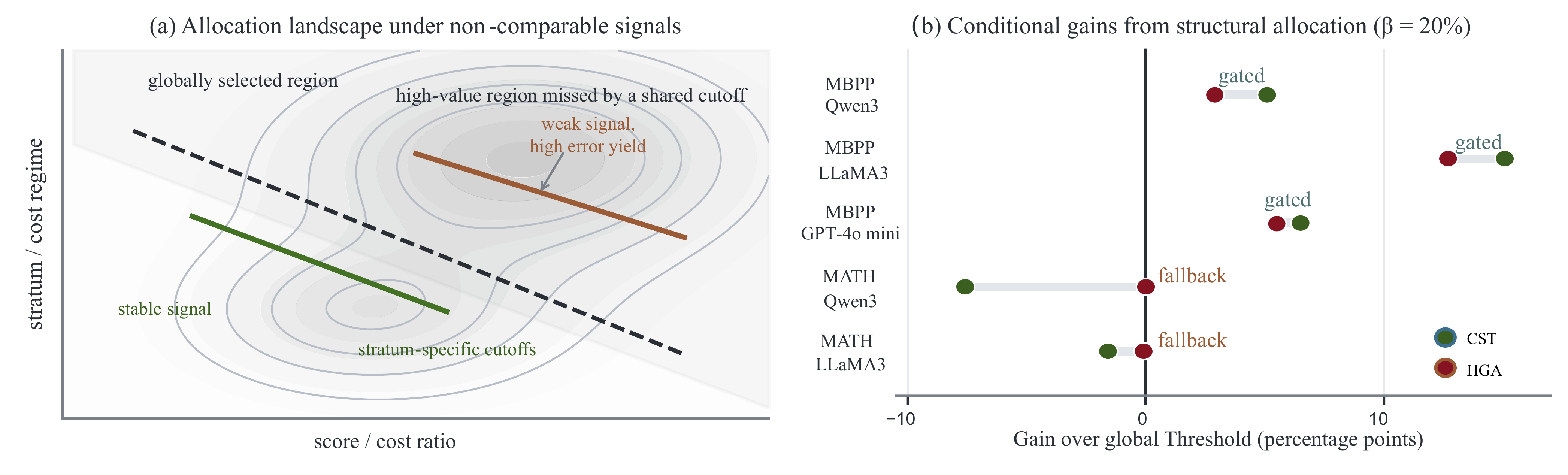}
\caption{Core mechanism and finding. \textbf{Left:} a globally shared threshold treats equal score-cost ratios as comparable, even when the proxy has different decision value across stable and weak-signal strata; stratified thresholds decouple these regimes. \textbf{Right:} at $\beta{=}20\%$, CST exposes when structure helps, while HGA retains positive stratification gains and falls back toward Threshold when fixed stratification is harmful.}
\label{fig:intro}
\end{figure*}

To answer this question, we need to distinguish three explanations that look similar in aggregate: weak signals, unstable online optimization, and feedback-structure mismatch. We therefore test a controlled intervention hierarchy that progressively relaxes global sharing. Threshold is the static global reference; MP-Adapt tests stronger global optimization; MP-Strat reduces optimization noise through partial stratification; and CST removes learning while fully stratifying allocation. CST is deliberately weak: if it beats a stronger globally shared learner, the failure is unlikely to be solved by optimizer sophistication alone. We then convert this diagnostic into Heterogeneity-Gated Allocation (HGA): a warm-up comparability test decides whether to trust the global threshold or open a cost-stratified gate. Figure~\ref{fig:intro} summarizes this mechanism and the resulting empirical pattern.

Across MBPP and MATH with Qwen3-8B, LLaMA3-8B, and GPT-4o-mini, MP-Strat provides partial recovery ($+2.8$pp) by reducing optimization noise, while CST achieves the largest gains ($+3.4$--$+17.0$pp) through structural decoupling in aligned settings. HGA keeps the diagnostic nature of CST but avoids treating stratification as universal: it matches Threshold on Qwen3-MATH, where CST underperforms, while retaining most CST gains on MBPP.

We theoretically analyze this phenomenon under bounded score densities, a locally Lipschitz quantile functional, and an explicit local location model. We show that when signal reliability varies across cost strata, a single global quantile can induce a structural allocation distortion whose bound grows with cross-stratum signal-quality dispersion. This analysis explains why improving a globally shared optimizer need not remove the observed failure.

The main contributions are as follows:
\begin{itemize}

\item We formalize global signal comparability as a hidden assumption in selective-compute decisions and identify its failure as a structural limitation of LLM resource allocation.

\item We develop a controlled diagnostic hierarchy (Threshold $\rightarrow$ MP-Adapt $\rightarrow$ MP-Strat $\rightarrow$ CST) and HGA, a gated allocation rule that selects global or stratified allocation based on warm-up comparability evidence.

\item We bound global-allocation distortion by cross-stratum discriminability dispersion and show across MBPP, MATH, and three model families that structure-aware allocation helps only when feedback is not comparable, motivating gated rather than unconditional stratification.

\end{itemize}

\section{Related Work}  
\label{sec:related}

\paragraph{Uncertainty in large language models.}
Uncertainty estimation in LLMs has been widely studied through token probabilities, logit margins, semantic consistency, and verbalized confidence \citep{guo2017calibration,kadavath2022language,kuhn2023semantic,farquhar2024detecting,huang2024uncertainty}. Classical selective prediction \citep{el2010foundations,geifman2017selective} and conformal prediction \citep{angelopoulos2021gentle,cherian2024large} offer principled frameworks for abstention with coverage guarantees. Modern instruction-tuned LLMs increasingly use uncertainty as a ranking signal for selective computation. Most existing approaches assume relatively stable uncertainty quality across inputs, enabling globally shared ranking policies. However, recent studies show that uncertainty quality can degrade substantially under distribution shift, task complexity, and instruction tuning \citep{huang2024uncertainty,tao2025revisiting,xiong2024can}, challenging the robustness of global ranking policies.

\paragraph{Resource allocation in LLM systems.}
LLM inference is increasingly framed as a resource-allocation problem. Self-consistency \citep{wang2022self}, best-of-N sampling, process-supervised verifiers, and test-time scaling \citep{cobbe2021training,snell2024scaling} allocate computation using globally shared ranking or reward signals. Related routing and tool-use systems similarly decide where costly capabilities should be invoked. Signal reliability, however, varies with input complexity and latent task difficulty \citep{manvi2024adaptive,damani2025learning,zhao2025t2}. We isolate the allocation consequence of this heterogeneity: a shared proxy can be globally useful while still being locally incomparable for budgeted decisions.

\paragraph{Execution-based verification.}
Execution-guided decoding, self-refinement, and verifier-based methods leverage runtime feedback to improve generation quality \citep{li2022competition,lightman2024let,madaan2023self,ni2023lever}. While effective for refining individual outputs, these works largely overlook the question of \emph{which} candidates to verify under limited computational budgets \citep{chen2021evaluating,jimenez2024swe}. Given that execution costs are high and highly variable across samples, principled prioritization via uncertainty signals becomes critical.

\paragraph{Budgeted verification and adaptive allocation.}
Budgeted LLM evaluation methods typically rely on uncertainty signals for global ranking or adaptive thresholding under fixed compute constraints \citep{fang2026inference,fan2026timebill}, or route queries across models to reduce cost \citep{chen2023frugalgpt,ong2024routellm}. These approaches assume uniform uncertainty quality across inputs. In practice, however, signal quality often degrades markedly on complex instances \citep{huang2024uncertainty,sharma2025assessing,tao2025revisiting}. 
Unlike prior work assuming uniform signal quality, we study the structural limitations induced by heterogeneous uncertainty through a controlled intervention hierarchy and cost-stratified verification policies. We draw connections to group-conditional selective prediction and per-group thresholding in the multicalibration literature 
\citep{hebert2018multicalibration,detommaso2024multicalibration}. Recent work has further extended these ideas to multi-group risk control with formal guarantees \citep{zhang2024fair}.

\paragraph{Constrained online decision making.}
Our setting is also closely related to constrained and budgeted bandits, bandits with knapsack-style resource constraints~\cite{badanidiyuru2018bandits, agrawal2014bandits}, and online resource allocation~\cite{badanidiyuru2014resourceful, balseiro2020dual}. These literatures primarily ask how to optimize reward or regret while respecting scarce-resource constraints. We study a complementary mechanism-level failure mode: the proxy used to allocate the resource can have systematically different reliability across observable cost regimes, so a globally shared learner may optimize a structurally conflated objective. MP-Adapt is therefore not presented as a new bandit algorithm; it is an optimization-focused control used to test whether stronger global adaptation resolves this misspecification.

\section{Problem Setup and Methods}
\label{sec:method}

Our goal is to diagnose when an allocation policy fails because its feedback structure is misaligned with the operating regimes it serves. We therefore express all policies within the same cost-aware decision framework and vary only the degree of sharing across strata. This design makes the comparison a mechanism test: if a weak structure-aligned policy outperforms a stronger globally shared learner, the bottleneck is not merely optimizer capacity.

\subsection{Problem Formulation}

Consider a set of $N$ samples. Each sample $i$ is associated with a binary error label $y_i \in \{0,1\}$, a proxy signal vector $\mathbf{h}_i \in \mathbb{R}^d$, and a verification cost $c_i > 0$ corresponding to executing slow feedback (e.g., unit tests).

At each step $t$, the policy observes proxy signal $\mathbf{h}_t$ and observable cost $\hat{c}_t$, and decides whether to trigger verification $a_t \in \{0,1\}$. The system is subject to a budget constraint:
\[
\sum_{t=1}^N a_t \cdot c_t \leq \beta \sum_{t=1}^N c_t,
\]
where $\beta \in (0,1)$ is the fraction of full-verification cost available. Policies obey the same budget cap and are evaluated by hit rate:
\[
\mathrm{HitRate} = \frac{\sum_t a_t \cdot y_t}{\sum_t a_t}.
\]
Because this ratio can favor under-utilization, we also report audit rates (Section~\ref{sec:cst_results}) and analyze discovered error mass through $V(\pi)$ (Section~\ref{sec:theory}).

\subsection{Unified Decision Framework and Proxy Signals}

All policies are expressed via a cost-aware scoring rule:
\begin{equation}
a_t = \mathbf{1}\!\left(s_\theta(\mathbf{h}_t) > \lambda \cdot \hat{c}_t\right),
\label{eq:unified}
\end{equation}
where $s_\theta(\cdot)$ maps proxy signals to scores and $\lambda$ controls budget pressure. Policies differ along two axes: whether the score is learned or fixed, and whether allocation is global or stratified.

All policies operate on the same token-level uncertainty proxies. Let $m_\ell = \log P(x_\ell^{(1)}) - \log P(x_\ell^{(2)})$ denote the logit margin at position $\ell$; we extract three complementary signals:
\begin{align*}
h_t^{(1)} &= -\frac{1}{L}\sum_{\ell=1}^L \log P(x_\ell \mid x_{<\ell}), \\
h_t^{(2)} &= -\min_\ell m_\ell, \\
h_t^{(3)} &= \frac{1}{L}\sum_{\ell=1}^L \exp\left(-\mathrm{clip}(m_\ell, -5, +\infty)\right).
\end{align*}
These capture average uncertainty ($h_t^{(1)}$), worst-case token confidence ($h_t^{(2)}$), and margin dispersion ($h_t^{(3)}$). The composite score is $s_\theta(\mathbf{h}_t) = \boldsymbol{\theta}^\top \mathbf{h}_t$, where $\mathbf{h}_t = [h_t^{(1)}, h_t^{(2)}, h_t^{(3)}]$.

\subsection{Global Policies: The Limits of Shared Adaptation}
\label{sec:global_policies}
We begin with globally shared policies, to isolate whether underperformance stems from insufficient learning or structural limitations of global adaptation.

\paragraph{Threshold.}
Threshold is the static global reference: it uses a fixed scoring function $s(\mathbf{h}_t) = \mathbf{w}_0^\top \mathbf{h}_t$ with a running quantile over signal-cost ratios:
\begin{align*}
a_t &= \mathbf{1}\!\left(\frac{s(\mathbf{h}_t)}{\hat{c}_t} > \tau_t\right), \\
\tau_t &= \mathrm{Quantile}_{1-\beta}\!\left\{\frac{s(\mathbf{h}_k)}{\hat{c}_k}\right\}_{k<t}.
\end{align*}
\paragraph{MP-Adapt: Global online learning.}
MP-Adapt tests whether stronger global adaptation helps, using a mirror-prox learner~\citep{nemirovski2004prox,fang2022online} that jointly optimizes $\boldsymbol{\theta}_t$ and a dual budget multiplier $\lambda_t$.

At step $t$, utility is:
\[
u_t = \boldsymbol{\theta}_t^\top \mathbf{h}_t - \lambda_t \hat{c}_t,
\]
and verification is triggered when $u_t > 0$.

Given outcome $y_t$, the scoring parameters are updated via exponentiated gradient~\citep{kivinen1997exponentiated}: 
\[
\mathrm{EG}(\boldsymbol{\theta}, \mathbf{g}, \eta)_k =
\frac{\theta_k \exp(\eta g_k)}{\sum_j \theta_j \exp(\eta g_j)},
\]
with gradient:
\[
\nabla_\theta J_t = (y_t - \sigma(u_t)) \cdot \mathbf{h}_t,
\]
where $\sigma$ is a sigmoid surrogate used only in the gradient (the decision rule itself is deterministic).
The dual variable is updated on every step:
\[
\lambda_{t+1} = \max\!\left(0,\; \lambda_t + \eta_\lambda (a_t \cdot c_t - \beta/N)\right).
\]
It represents the strongest globally shared learner; if global sharing is the bottleneck, per-stratum control should recover performance.

\paragraph{MP-Strat: Partial stratification.}
MP-Strat tests optimization instability by maintaining independent dual variables $\lambda_k$ per cost stratum and a teacher--student EMA~\citep{he2020momentum}:
\[
\boldsymbol{\theta}^T_k \leftarrow \gamma \boldsymbol{\theta}^T_k + (1-\gamma)\boldsymbol{\theta}^S_k,
\quad \gamma = 0.9.
\]
The teacher $\boldsymbol{\theta}^T_k$ is used for decisions; the student $\boldsymbol{\theta}^S_k$ is updated online.

\subsection{Structural Partitioning Without Learning}
\label{sec:cst}
CST tests structural alignment without online learning. It partitions samples into $K=4$ strata by ex-ante cost $\hat{c}_t$ and applies independent running quantiles:
\begin{align*}
a_t &= \mathbf{1}\!\left(\frac{s(\mathbf{h}_t)}{\hat{c}_t} > \tau_{k(t),t}\right), \\
\tau_{k,t} &= \mathrm{Quantile}_{1-\beta}\!\left\{
\frac{s(\mathbf{h}_j)}{\hat{c}_j} \mid k(j)=k
\right\}_{j<t}.
\end{align*}

Here $k(t)$ is the cost stratum index. CST uses Threshold's fixed score and the same online information stream---no additional labels or gradient updates. Its only difference from Threshold is structural, so their gap isolates stratification.

\subsection{Heterogeneity-Gated Allocation}
\label{sec:hga}
HGA turns the diagnostic into a policy-selection rule rather than a post-hoc oracle over methods. On a warm-up split, it estimates stratum-level signal quality $\rho_k=\mathrm{corr}(s,y\mid k)$, base error rates, and the validation hit rates of Threshold and CST. The gate opens only when CST is not worse on warm-up feedback and the cost partition exhibits both discriminability dispersion and base-rate alignment:
\[
\mathrm{Var}_k(\rho_k)\geq \epsilon_H,\qquad
\max_k \hat{p}_k-\min_k \hat{p}_k\geq \epsilon_p .
\]
If the gate opens, HGA deploys CST; otherwise it deploys the global Threshold policy. HGA uses the same score initialization and budget accounting as the candidate policies; only the deployment choice is gated by warm-up comparability evidence. Thus HGA is not a stronger optimizer than CST. It is a guard against blind stratification: the same diagnostic that identifies comparability failure also decides whether structural specialization should be used.

\begin{table*}[t]
\centering
\caption{Main results: mean hit rate across 10 seeds (heterogeneous costs). Bold: best non-Oracle per row. The rightmost column reports paired mean CST$-$Threshold gain with 95\% CI in brackets.}
\label{tab:main}
\begin{tabular*}{\textwidth}{@{\extracolsep{\fill}}llcccccccc}
\toprule
Setting & $\beta$ & Rand & Thr & CST & C+M & MP-A & MP-S & Oracle & $\Delta$(CST$-$Thr) [95\% CI] \\
\midrule
\multirow{4}{*}{\shortstack[l]{MBPP\\Qwen3\\err=.558\\$\rho$=.178}}
 & 10\% & .551 & .636 & \textbf{.671} & .622 & .597 & .656 & 1.00 & $+3.4$ $[-0.2,+7.0]$ \\
 & 20\% & .546 & .610 & \textbf{.661} & .647 & .619 & .634 & 1.00 & $+5.1$ $[+3.6,+6.6]$ \\
 & 30\% & .550 & .605 & \textbf{.661} & .638 & .616 & .599 & 1.00 & $+5.6$ $[+4.3,+7.0]$ \\
 & 50\% & .562 & .588 & \textbf{.631} & .615 & .569 & .600 & 1.00 & $+4.3$ $[+3.3,+5.3]$ \\
\midrule
\multirow{4}{*}{\shortstack[l]{MBPP\\LLaMA3\\err=.644\\$\rho$=.301}}
 & 10\% & .630 & .712 & \textbf{.882} & .844 & .689 & .837 & 1.00 & $+17.0$ $[+12.8,+21.1]$ \\
 & 20\% & .638 & .686 & \textbf{.837} & .795 & .693 & .794 & 1.00 & $+15.1$ $[+13.2,+17.0]$ \\
 & 30\% & .642 & .693 & \textbf{.803} & .755 & .692 & .763 & 1.00 & $+11.0$ $[+9.2,+12.9]$ \\
 & 50\% & .649 & .675 & \textbf{.728} & .693 & .673 & .704 & 1.00 & $+5.3$ $[+4.1,+6.5]$ \\
\midrule
\multirow{4}{*}{\shortstack[l]{MBPP\\GPT-4o-mini\\err=.528\\$\rho$=.174}}
 & 10\% & .513 & .511 & \textbf{.592} & .592 & .524 & .552 & .99 & $+8.1$ $[+4.5,+11.7]$ \\
 & 20\% & .527 & .545 & .610 & \textbf{.610} & .553 & .579 & .99 & $+6.5$ $[+4.3,+8.7]$ \\
 & 30\% & .533 & .547 & .614 & \textbf{.620} & .563 & .578 & 1.00 & $+6.7$ $[+5.6,+7.8]$ \\
 & 50\% & .539 & .578 & \textbf{.612} & .601 & .580 & .576 & 1.00 & $+3.4$ $[+2.5,+4.3]$ \\
\midrule
\multirow{4}{*}{\shortstack[l]{MATH\\Qwen3\\err=.247\\$\rho$=.153}}
 & 10\% & .242 & \textbf{.524} & .322 & .308 & .356 & .288 & 1.00 & $-20.2$ $[-22.3,-18.1]$ \\
 & 20\% & .246 & \textbf{.384} & .308 & .303 & .310 & .285 & 1.00 & $-7.6$ $[-8.4,-6.8]$ \\
 & 30\% & .248 & \textbf{.320} & .301 & .296 & .284 & .265 & .94 & $-1.9$ $[-2.4,-1.4]$ \\
 & 50\% & .248 & .264 & \textbf{.283} & .272 & .250 & .252 & .53 & $+2.0$ $[+1.6,+2.3]$ \\
\midrule
\multirow{4}{*}{\shortstack[l]{MATH\\LLaMA3\\err=.613\\$\rho$=.201}}
 & 10\% & .605 & \textbf{.848} & .799 & .744 & .734 & .730 & 1.00 & $-4.9$ $[-6.7,-3.2]$ \\
 & 20\% & .608 & \textbf{.774} & .758 & .718 & .681 & .699 & 1.00 & $-1.6$ $[-2.4,-0.8]$ \\
 & 30\% & .615 & .719 & \textbf{.734} & .698 & .653 & .671 & 1.00 & $+1.5$ $[+1.1,+2.0]$ \\
 & 50\% & .616 & .655 & \textbf{.693} & .668 & .615 & .632 & 1.00 & $+3.7$ $[+3.4,+4.1]$ \\
\bottomrule
\end{tabular*}
\end{table*}

\section{Theoretical Analysis}
\label{sec:theory}

\subsection{Setup}

Samples are drawn from $K$ latent strata with mixture weights
$w_k$. Each stratum $k$ has signal-error correlation $\rho_k =
\mathrm{corr}(s, y \mid k)$, computed as the point-biserial correlation between continuous scores and binary error labels (significance: two-sided $t$-test), and signal variance $\sigma_k^2$.
Let $z=s/\hat{c}$ denote the cost-normalized ranking score. For the theoretical comparison, define the \emph{error-discovery value} of a policy as
\[
V(\pi)=\mathbb{E}[y\cdot\mathbf{1}\{\pi\text{ verifies the sample}\}],
\]
which measures discovered error mass under a fixed population budget. This differs from the empirical hit-rate ratio; we use it because it admits a transparent decomposition of allocation distortion. The theorem therefore predicts when global and stratified allocation can diverge in error mass; the empirical hit-rate and audit-rate analyses then test whether that divergence appears as useful auditing precision, increased coverage, or harmful over-allocation.

\begin{definition}[Global Quantile Policy]
$\pi_{\mathrm{global}}$ selects the top $\beta$-fraction of samples ranked by $s/\hat{c}$.
\end{definition}
One might expect that a sufficiently powerful global optimizer could learn to account for stratum-specific signal quality---effectively discovering the stratification from data. The following result isolates why a global quantile remains structurally sensitive to heterogeneity even with perfect knowledge of the score distribution: its allocation distortion is controlled by cross-stratum discriminability dispersion.
\begin{proposition}[Heteroskedastic Allocation-Distortion Bound]
\label{thm:gap}
Define per-stratum discriminability $\delta_k$ as the standardized separation between error and correct-sample distributions of $z$. Assume: \textbf{(A1)} the error-weighted density $g_k(z)=p_k f_k(z\mid y{=}1)$ satisfies $\sup_z g_k(z)\leq M$; \textbf{(A2)} the relevant global and stratum quantile maps are locally $\ell^{-1}$-Lipschitz, equivalently the corresponding score densities are lower bounded by $\ell>0$ on the quantile neighborhood; and \textbf{(A3)} as a concrete local model, strata share error prevalence and scale and differ through the Gaussian location-separation parameter $\delta_k$. Then:
\begin{equation}
\begin{split}
&\bigl|V(\pi_{\mathrm{strat}})-V(\pi_{\mathrm{global}})\bigr| \\
&\quad \leq \frac{2Mp}{\ell\sqrt{2\pi}}
\sum_k w_k\,|\delta_k-\bar{\delta}|,
\end{split}
\label{eq:gap}
\end{equation}
where $p$ is the shared error prevalence and $\bar{\delta}=\sum_k w_k\delta_k$. Under (A3), the bound vanishes when $\delta_k=\bar{\delta}$ for all $k$.
\end{proposition}
\begin{figure*}[t]
  \centering
  \includegraphics[width=\textwidth]{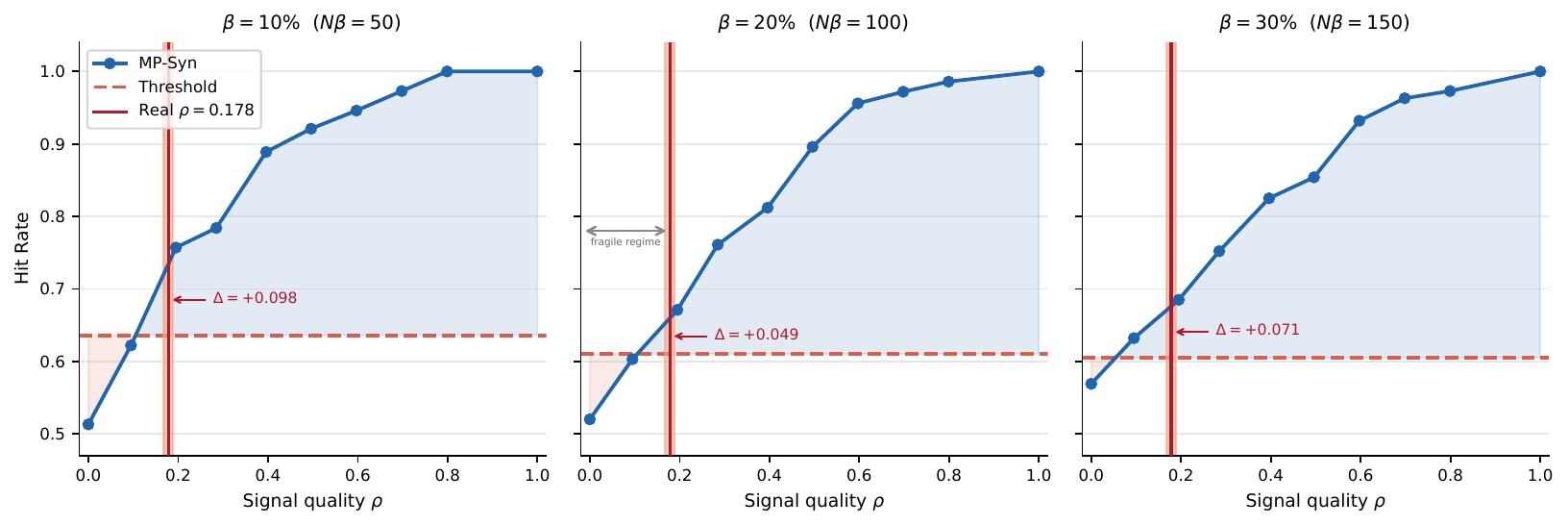}
  \caption{Signal strength sweep under homogeneous synthetic signals ($h^{\text{syn}} = \rho\tilde{y} + \sqrt{1{-}\rho^2}\,\varepsilon$, uniform $\rho_k \equiv \rho$). MP-Syn (blue curve) vs.\ Threshold (red dashed); blue diamond: homogeneous MP-Syn at $\rho{=}0.178$; red cross: actual MP-Adapt. The gap indicates efficiency loss from heterogeneous signal structures.}
  \label{fig:sweep}
\end{figure*}

(A3)'s homogeneous prevalence isolates discriminability heterogeneity from base-rate heterogeneity, analyzed separately as Source~B below.

Under the equal-variance location model, $\delta_k$ is monotonically related to the point-biserial correlation $\rho_k = \mathrm{corr}(s, y \mid k)$; for weak separation, $\delta_k \approx \rho_k / \sqrt{p_k(1-p_k)}$. Thus cross-stratum variation in $\rho_k$ provides an observable, local diagnostic for the dispersion term. The supplementary proof derives the explicit constant by composing a quantile-perturbation bound with the bounded error-mass density in (A1).

Proposition~\ref{thm:gap} characterizes allocation distortion induced by heterogeneous discriminability (\textbf{Source A}). The bound is diagnostic rather than a universal lower bound on CST's gain: it controls how far global and stratified allocation values can diverge under a transparent local model, but does not determine the sign of the difference or directly bound the empirical hit-rate ratio. This is distinct from a second empirical mechanism: CST also reallocates budget toward high-error strata whose uncertainty signals are weak but whose base error rates are substantially elevated (\textbf{Source B}). For example, in Q4 of MBPP Qwen3, $\rho_{\mathrm{Q4}}{=}0.157$ (n.s.) yet the error rate is 0.714; CST increases audit coverage from 0.151 to 0.191, trading precision for recall in error-dense regions. Together, the proposition and this coverage analysis explain why structural partitioning helps in strongly heterogeneous settings while correctly allowing negative gains when heterogeneity is weak or poorly aligned.

\textbf{Cost as a signal-quality proxy.} Per-stratum $\rho_k$ values (Figure~\ref{fig:quartile_rho}) show that ex-ante cost identifies heterogeneity regimes across all three models. Despite opposite directional patterns---Qwen3 and GPT-4o-mini degrade in Q4 while LLaMA3 rises monotonically---CST improves over Threshold in all three MBPP cases.

\section{Experiments}
\label{sec:experiment}

\subsection{Setup}

\paragraph{Datasets and Models.}
We evaluate Qwen3-8B \citep{qwen3}, LLaMA3-8B \citep{grattafiori2024llama3herdmodels}, and GPT-4o-mini \citep{openai2024gpt4ocard} on MBPP \citep{austin2021program} ($N{=}500$; unit-test feedback) and MATH \citep{hendrycksmath2021} ($N{=}5000$; answer matching). Error rates are 55.8\%/24.7\% for Qwen3 on MBPP/MATH, 64.4\%/61.3\% for LLaMA3, and 52.8\% for GPT-4o-mini on MBPP.

\paragraph{Costs, Noise, and Baselines.}
Decisions use an input-length-normalized proxy cost $\hat{c}_t$, while budget consumption uses output-length-normalized true cost $c_t$ (correlation $\approx0.3$). We add 10\% Bernoulli label noise and Gaussian proxy noise ($\sigma{=}0.1$). The diagnostic chain is Random, Threshold, MP-Adapt, MP-Strat, CST, and Oracle; Table~\ref{tab:main} also reports CST+Memory, with details in the supplementary material.

\paragraph{Hyperparameters.}
We use $\eta_\theta{=}0.5$ with decay $0.002$, EMA momentum $\gamma{=}0.9$, a $\lambda_k$ floor of $0.1$, and budget-scaled $\eta_\lambda$. Threshold and CST fix $\mathbf{w}_0$ from the first $T_0{=}50$ samples. HGA keeps this same score initialization and uses a separate 50\% calibration split for the gate, with $\epsilon_H{=}0.002$, $\epsilon_p{=}0.03$, and nonnegative warm-up gain required before opening the CST gate. All policies use the same stream order, proxy signals, noisy labels, and budget accounting; we report Student-$t$ 95\% CIs across 10 seeds.

\subsection{Main Result: A Hierarchy of Bottlenecks}
\label{sec:main_obs}

Table~\ref{tab:main} tests whether failure of global signal comparability, rather than insufficient adaptation, explains global-policy underperformance. If so, the deliberately weaker but structurally aligned CST should outperform both global policies (full results in the supplementary material).

MP-Adapt fails to reliably outperform Threshold, while MP-Strat's partial recovery shows that optimization instability is present but does not explain the full gap. CST performs best when cross-stratum heterogeneity is strong: less learning yields better allocation once the decision structure matches the feedback regime.

We isolate three potential sources of gain---signal heteroskedasticity, cost heterogeneity, and optimization noise. Homogeneous-cost controls in the supplementary material rule out cost heterogeneity alone, while MP-Strat shows that optimization noise is insufficient to explain the observed differences; together, these controls identify cross-stratum signal heteroskedasticity as the most consistent explanation among the tested mechanisms.

Across MBPP settings, CST improves over Threshold with CIs excluding zero (Table~\ref{tab:main}). At $\beta{=}20\%$, the gain reaches $+5.1$pp (95\% CI $[+3.6, +6.6]$) for Qwen3, $+15.1$pp $[+13.2, +17.0]$ for LLaMA3, and $+6.5$pp $[+4.3, +8.7]$ for GPT-4o-mini. In contrast, Qwen3-MATH exhibits a negative CST gain ($-7.6$pp at $\beta{=}20\%$, 95\% CI $[-8.4, -6.8]$). This conditional behavior is consistent with our diagnosis: stratification is useful only when the chosen partition aligns with consequential heterogeneity.

In a calibration reanalysis of the gate, HGA stays close to CST in the MBPP settings where cost-aligned heterogeneity is useful, but recovers Threshold on Qwen3-MATH, where fixed stratification underperforms. At $\beta{=}20\%$, HGA obtains .653/.804/.591 on MBPP Qwen3/LLaMA3/GPT-4o-mini and .385/.774 on MATH Qwen3/LLaMA3, respectively (full per-budget results in the supplementary material). This is the intended role of the gate: not to dominate the best fixed policy in every setting, but to avoid committing to stratification when comparability evidence is weak.

Figure~\ref{fig:hga_gain} summarizes this boundary condition across budgets.
\begin{figure*}[t]
  \centering
  \includegraphics[width=\textwidth]{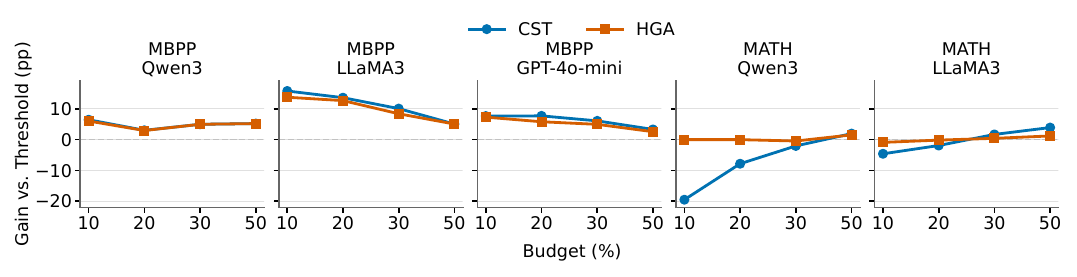}
  \caption{CST and HGA gain over Threshold ($\Delta$ hit rate, pp) across budgets $\beta$ in the calibration reanalysis.}
  \label{fig:hga_gain}
\end{figure*}

\begin{figure}[t]
  \centering
  \includegraphics[width=\columnwidth]{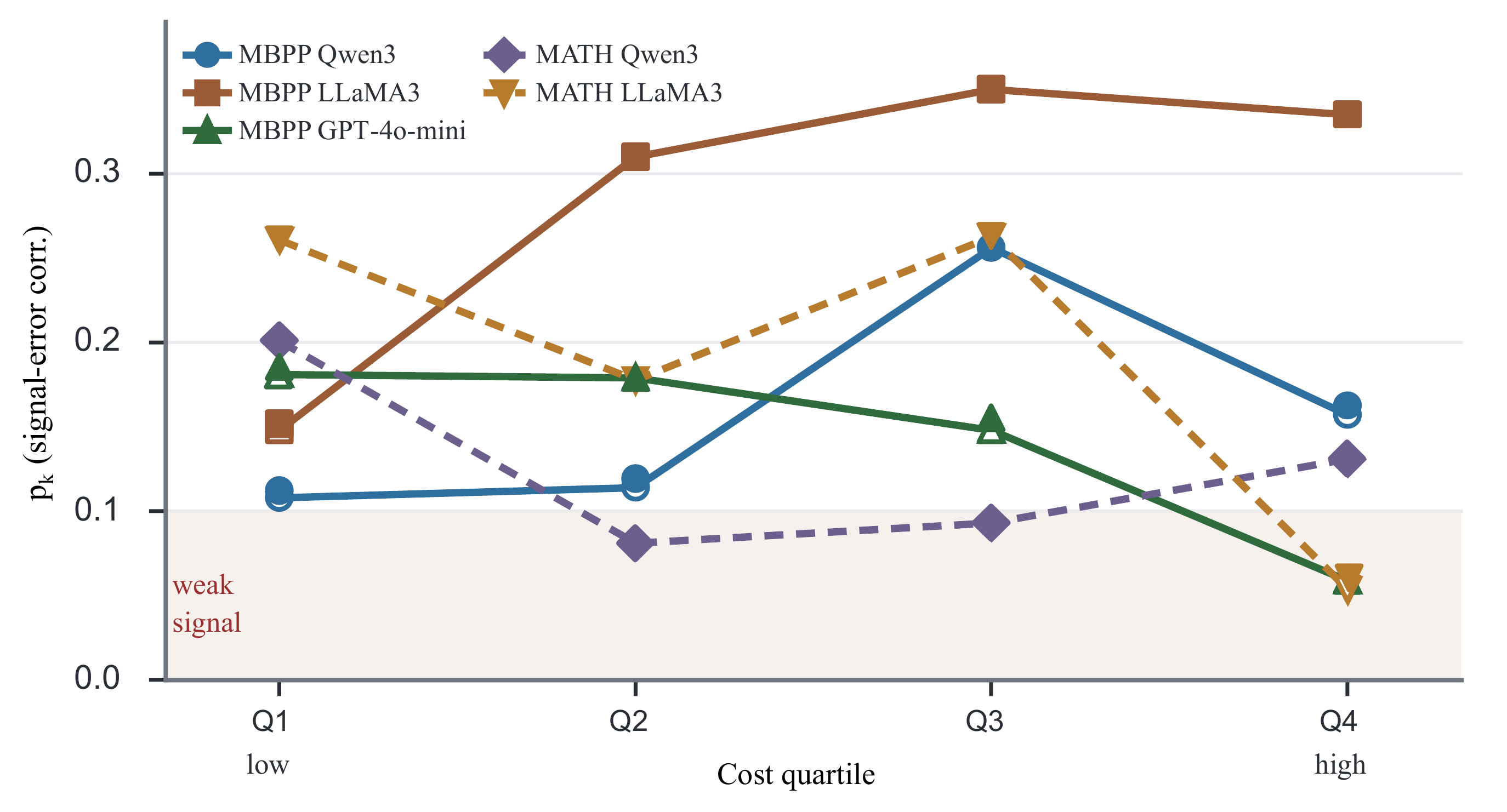}
  \caption{Per-stratum $\rho_k$ across cost quartiles. Filled markers indicate $p{<}.05$ and open markers indicate n.s.; shaded region marks near-random signal ($\rho_k{<}0.10$).}
\label{fig:quartile_rho}
\end{figure}

\subsection{Signal Quality and Heteroskedasticity}
\label{sec:signal_sweep}
We now diagnose the bottleneck via three hypotheses: weak signal, optimization instability, or cross-stratum heterogeneity.
\paragraph{Signal strength analysis.}
We construct synthetic signals:
\[
h^{\mathrm{syn}} = \rho \cdot \tilde{y} + \sqrt{1-\rho^2} \cdot \varepsilon, \quad \varepsilon \sim \mathcal{N}(0,1).
\]

Figure~\ref{fig:sweep} shows MP-Syn improves over Threshold when $\rho \gtrsim 0.20$. At the observed $\rho{=}0.178$, MP-Syn remains above Threshold under homogeneous assumptions, while the actual MP-Adapt result is below---indicating additional inefficiency from heterogeneous signal structures.

\paragraph{Within-stratum heteroskedasticity.}
MBPP Qwen3 exhibits substantial per-stratum variation: Q4 combines the highest error rate (0.714) with weak signal ($\rho{=}0.157$, n.s.), while only Q3 has significant discriminability ($\rho{=}0.256$, $p{<}.01$).

Figure~\ref{fig:quartile_rho} extends this analysis across all settings. MBPP models exhibit strong heterogeneity with distinct per-model patterns; MATH exhibits weaker and less structured variation, consistent with its smaller CST gains (Proposition~\ref{thm:gap}). Because MP-Adapt aggregates gradients globally, low-quality strata dilute informative gradients from high-signal regions.

\paragraph{Hypothesis tests.} Three ablations support this diagnosis. (1)~\textit{Partial feedback}: restricting updates to verified samples changes performance by $<0.5$pp, making feedback sparsity an unlikely primary explanation. (2)~\textit{Oracle signal}: replacing proxy signals with true labels yields perfect hit rate, showing that the optimization framework is sufficient given reliable signals. (3)~\textit{MP-Strat}: stabilization partially recovers ($+2.8$pp), but the remaining gap to CST ($-1.9$pp) points to structural heterogeneity as the largest remaining limitation among the tested hypotheses.

\paragraph{Diagnosing comparability failure.}
Optimizer weakness predicts that stronger global learning closes the stratification gap. Comparability failure instead predicts varying stratum-level signal quality, partial MP-Strat recovery, and a weak aligned intervention outperforming global adaptation. Negative-control partitions and Qwen3-MATH support this distinction: structural alignment helps only when the partition captures consequential feedback regimes.

\subsection{A Weak Structural Intervention}
\label{sec:cst_results}
Having diagnosed comparability failure, we use CST as a deliberately weak intervention to examine how structural alignment changes allocation.
Table~\ref{tab:cst_quartile} illustrates two complementary sources of CST's improvement.

\begin{table}[t]
\centering
\small
\caption{CST vs.\ Threshold per-stratum behavior ($\beta=20\%$).}
\label{tab:cst_quartile}
\begin{tabular}{lcccccc}
\toprule
 & \multicolumn{2}{c}{Audit rate} & \multicolumn{2}{c}{Hit rate} & Error \\
\cmidrule(lr){2-3}\cmidrule(lr){4-5}
Stratum & Thr & CST & Thr & CST & rate \\
\midrule
Q1 & .306 & .189 & .412 & .429 & .423 \\
Q2 & .216 & .209 & .586 & .571 & .500 \\
Q3 & .186 & .217 & .625 & .714 & .581 \\
Q4 & .151 & .191 & .947 & .833 & .714 \\
\bottomrule
\end{tabular}
\end{table}

\textbf{Source 1 (precision gain in high-signal strata):} In Q3 ($\rho=0.256$, $p{<}.01$), CST increases audit rate from 0.186 to 0.217 while hit rate rises from 0.625 to 0.714 ($+8.9$pp).

\textbf{Source 2 (coverage gain in low-signal strata):} In Q4 ($\rho=0.157$, n.s., error rate 0.714), CST increases audit rate from 0.151 to 0.191 ($+26.5\%$), trading precision for recall: hit rate drops from 0.947 to 0.833, but more total errors are found. Despite near-random signals, the high base error rate makes increased coverage effective.

Negative-control stratification results in the supplementary material show that cost is the only tested ex-ante criterion that improves over global Threshold; random partitions and noisy signal-strength bins degrade performance. Per-stratum split-conformal calibration further confirms that stratification, rather than CST-specific optimization, drives most gains.

\subsection{Cross-Task and Cross-Model Analysis}
\label{sec:cross_model}
Finally, we characterize the boundary conditions for structural partitioning.
Figure~\ref{fig:hga_gain} shows that MBPP gains are consistently positive, whereas MATH is conditional on model and budget. HGA follows these boundary conditions rather than treating stratification as universally beneficial; full gate decisions appear in the supplementary material. Table~\ref{tab:main} reveals three cross-task patterns.

This evaluation frames HGA as policy selection rather than post-hoc model choice: replace a global comparator with stratum-local comparators only when warm-up feedback supports specialization.

\textbf{(1) CST gains depend on both cost-error correlation and heteroskedasticity.}
LLaMA3-MATH ($r=0.309$) shows CST gains at $\beta\geq30\%$, while Qwen3-MATH ($r=0.193$) underperforms. GPT-4o-mini-MBPP achieves higher gains despite lower global correlation due to stronger signal collapse in high-cost strata.

\textbf{(2) Optimization stabilization has conditional benefits.} MP-Strat improves over MP-Adapt only in moderate-to-high correlation regimes; in low-correlation settings (e.g., Qwen3-MATH), stabilization can reduce performance.

\textbf{(3) Heteroskedasticity is structural, not model-specific.} Across MBPP, signal degradation occurs in opposite directions yet CST consistently improves. A heterogeneity index $H{=}\mathrm{Var}_k(\rho_k)$ captures this: Qwen3 ($H{=}0.0035$) vs.\ LLaMA3 ($H{=}0.0065$) yields $+5.1$ vs.\ $+15.1$pp; GPT-4o-mini ($H{=}0.0025$) achieves $+6.5$pp via extreme Q4 signal collapse.

\paragraph{A diagnostic protocol for allocation systems.}
Before increasing optimizer complexity, practitioners should measure conditional signal quality across observable strata and compare an aligned partition against both the global policy and negative-control partitions. Improvement only under the aligned partition indicates comparability failure; HGA operationalizes this test by gating stratification on warm-up evidence.

\section{Conclusion}
\label{sec:conclusion}
The intervention hierarchy separates weak signals, optimization instability, and structural heterogeneity: MP-Strat partially recovers performance, while CST exposes when structural alignment matters. HGA turns this diagnosis into a guarded policy choice. The broader principle is to test feedback comparability before trusting stronger global optimization: a failed adaptive allocator may indicate that equal proxy scores have different decision value across operating regimes.

\bibliography{custom}

@inproceedings{agrawal2014bandits,
  title={Bandits with concave rewards and convex knapsacks},
  author={Agrawal, Shipra and Devanur, Nikhil R},
  booktitle={Proceedings of the fifteenth ACM conference on Economics and computation},
  pages={989--1006},
  year={2014}
}

@inproceedings{badanidiyuru2014resourceful,
  title={Resourceful contextual bandits},
  author={Badanidiyuru, Ashwinkumar and Langford, John and Slivkins, Aleksandrs},
  booktitle={Conference on Learning Theory},
  pages={1109--1134},
  year={2014},
  organization={PMLR}
}

@article{badanidiyuru2018bandits,
  title={Bandits with knapsacks},
  author={Badanidiyuru, Ashwinkumar and Kleinberg, Robert and Slivkins, Aleksandrs},
  journal={Journal of the ACM (JACM)},
  volume={65},
  number={3},
  pages={1--55},
  year={2018},
  publisher={ACM New York, NY, USA}
}

@inproceedings{balseiro2020dual,
  title={Dual mirror descent for online allocation problems},
  author={Balseiro, Santiago and Lu, Haihao and Mirrokni, Vahab},
  booktitle={International Conference on Machine Learning},
  pages={613--628},
  year={2020},
  organization={PMLR}
}

@article{geifman2017selective,
  title={Selective classification for deep neural networks},
  author={Geifman, Yonatan and El-Yaniv, Ran},
  journal={Advances in neural information processing systems},
  volume={30},
  year={2017}
}

@inproceedings{guo2017calibration,
  title={On calibration of modern neural networks},
  author={Guo, Chuan and Pleiss, Geoff and Sun, Yu and Weinberger, Kilian Q},
  booktitle={International conference on machine learning},
  pages={1321--1330},
  year={2017},
  organization={PMLR}
}

@article{kadavath2022language,
  title={Language models (mostly) know what they know},
  author={Kadavath, Saurav and Conerly, Tom and Askell, Amanda and Henighan, Tom and Drain, Dawn and Perez, Ethan and Schiefer, Nicholas and Hatfield-Dodds, Zac and DasSarma, Nova and Tran-Johnson, Eli and others},
  journal={arXiv preprint arXiv:2207.05221},
  year={2022}
}

@article{el2010foundations,
  title={On the Foundations of Noise-free Selective Classification.},
  author={El-Yaniv, Ran and others},
  journal={Journal of Machine Learning Research},
  volume={11},
  number={5},
  year={2010}
}

@article{kuhn2023semantic,
  title={Semantic uncertainty: Linguistic invariances for uncertainty estimation in natural language generation},
  author={Kuhn, Lorenz and Gal, Yarin and Farquhar, Sebastian},
  journal={arXiv preprint arXiv:2302.09664},
  year={2023}
}

@article{wang2022self,
  title={Self-consistency improves chain of thought reasoning in language models},
  author={Wang, Xuezhi and Wei, Jason and Schuurmans, Dale and Le, Quoc and Chi, Ed and Narang, Sharan and Chowdhery, Aakanksha and Zhou, Denny},
  journal={arXiv preprint arXiv:2203.11171},
  year={2022}
}

@article{snell2024scaling,
  title={Scaling llm test-time compute optimally can be more effective than scaling model parameters},
  author={Snell, Charlie and Lee, Jaehoon and Xu, Kelvin and Kumar, Aviral},
  journal={arXiv preprint arXiv:2408.03314},
  year={2024}
}

@article{austin2021program,
  title={Program synthesis with large language models},
  author={Austin, Jacob and Odena, Augustus and Nye, Maxwell and Bosma, Maarten and Michalewski, Henryk and Dohan, David and Jiang, Ellen and Cai, Carrie and Terry, Michael and Le, Quoc and others},
  journal={arXiv preprint arXiv:2108.07732},
  year={2021}
}

@article{chen2021evaluating,
  title={Evaluating large language models trained on code},
  author={Chen, Mark and Tworek, Jerry and Jun, Heewoo and Yuan, Qiming and Pinto, Henrique Ponde De Oliveira and Kaplan, Jared and Edwards, Harri and Burda, Yuri and Joseph, Nicholas and Brockman, Greg and others},
  journal={arXiv preprint arXiv:2107.03374},
  year={2021}
}

@inproceedings{jimenez2024swe,
  title={Swe-bench: Can language models resolve real-world github issues?},
  author={Jimenez, Carlos E and Yang, John and Wettig, Alexander and Yao, Shunyu and Pei, Kexin and Press, Ofir and Narasimhan, Karthik},
  booktitle={International Conference on Learning Representations},
  volume={2024},
  pages={54107--54157},
  year={2024}
}

@article{li2022competition,
  title={Competition-level code generation with alphacode},
  author={Li, Yujia and Choi, David and Chung, Junyoung and Kushman, Nate and Schrittwieser, Julian and Leblond, R{\'e}mi and Eccles, Tom and Keeling, James and Gimeno, Felix and Dal Lago, Agustin and others},
  journal={Science},
  volume={378},
  number={6624},
  pages={1092--1097},
  year={2022},
  publisher={American Association for the Advancement of Science}
}

@article{madaan2023self,
  title={Self-refine: Iterative refinement with self-feedback},
  author={Madaan, Aman and Tandon, Niket and Gupta, Prakhar and Hallinan, Skyler and Gao, Luyu and Wiegreffe, Sarah and Alon, Uri and Dziri, Nouha and Prabhumoye, Shrimai and Yang, Yiming and others},
  journal={Advances in neural information processing systems},
  volume={36},
  pages={46534--46594},
  year={2023}
}

@article{hendrycksmath2021,
  title={Measuring Mathematical Problem Solving With the MATH Dataset},
  author={Dan Hendrycks and Collin Burns and Saurav Kadavath and Akul Arora and Steven Basart and Eric Tang and Dawn Song and Jacob Steinhardt},
  journal={NeurIPS},
  year={2021}
}

@inproceedings{huang2024uncertainty,
  title={Uncertainty in language models: Assessment through rank-calibration},
  author={Huang, Xinmeng and Li, Shuo and Yu, Mengxin and Sesia, Matteo and Hassani, Hamed and Lee, Insup and Bastani, Osbert and Dobriban, Edgar},
  booktitle={Proceedings of the 2024 Conference on Empirical Methods in Natural Language Processing},
  pages={284--312},
  year={2024}
}

@article{tao2025revisiting,
  title={Revisiting uncertainty estimation and calibration of large language models},
  author={Tao, Linwei and Yeh, Yi-Fan and Dong, Minjing and Huang, Tao and Torr, Philip and Xu, Chang},
  journal={arXiv preprint arXiv:2505.23854},
  year={2025}
}

@article{sharma2025assessing,
  title={Assessing correctness in llm-based code generation via uncertainty estimation},
  author={Sharma, Arindam and David, Cristina},
  journal={arXiv preprint arXiv:2502.11620},
  year={2025}
}

@inproceedings{xiong2024can,
  title={Can llms express their uncertainty? an empirical evaluation of confidence elicitation in llms},
  author={Xiong, Miao and Hu, Zhiyuan and Lu, Xinyang and Li, Yifei and Fu, Jie and He, Junxian and Hooi, Bryan},
  booktitle={International Conference on Learning Representations},
  volume={2024},
  pages={23650--23678},
  year={2024}
}

@article{manvi2024adaptive,
  title={Adaptive inference-time compute: Llms can predict if they can do better, even mid-generation},
  author={Manvi, Rohin and Singh, Anikait and Ermon, Stefano},
  journal={arXiv preprint arXiv:2410.02725},
  year={2024}
}

@inproceedings{zhao2025t2,
  title={T2: An Adaptive Test-Time Scaling Strategy for Contextual Question Answering},
  author={Zhao, Zhengyi and Zhang, Shubo and Wang, Zezhong and Wang, Huimin and Zhao, Yutian and Liang, Bin and Zheng, Yefeng and Li, Binyang and Wong, Kam-Fai and Wu, Xian},
  booktitle={Proceedings of the 2025 Conference on Empirical Methods in Natural Language Processing},
  pages={3731--3756},
  year={2025}
}

@article{fang2026inference,
  title={Inference-Time Budget Control for LLM Search Agents},
  author={Fang, Zhengru and Hu, Senkang Forest and Chang, Zhonghao and Guo, Yu and Tao, Yihang and Liu, Hongyao and Ruan, Mengzhe and Huang, Jun and Fang, Yuguang},
  journal={arXiv preprint arXiv:2605.05701},
  year={2026}
}

@inproceedings{fan2026timebill,
  title={Timebill: Time-budgeted inference for large language models},
  author={Fan, Qi and Zou, An and Ma, Yehan},
  booktitle={Proceedings of the AAAI Conference on Artificial Intelligence},
  volume={40},
  number={36},
  pages={30620--30628},
  year={2026}
}

@article{qwen3,
    title={Qwen3 Technical Report}, 
    author={Yang, An and Li, Anfeng and Yang, Baosong and Zhang, Beichen and Hui, Binyuan and others},
    journal = {arXiv preprint arXiv:2505.09388},
    year={2025}
}

@misc{openai2024gpt4ocard,
      title={GPT-4o System Card}, 
      author={OpenAI and Hurst, Aaron and Lerer, Adam and Goucher, Adam P. and Perelman, Adam and others},
      year={2024},
      eprint={2410.21276},
      archivePrefix={arXiv},
      primaryClass={cs.CL},
      url={https://arxiv.org/abs/2410.21276}, 
}

@misc{grattafiori2024llama3herdmodels,
      title={The Llama 3 Herd of Models}, 
      author={Grattafiori, Aaron and Dubey, Abhimanyu and Jauhri, Abhinav and Pandey, Abhinav and Kadian, Abhishek and others},
      year={2024},
      eprint={2407.21783},
      archivePrefix={arXiv},
      primaryClass={cs.AI},
      url={https://arxiv.org/abs/2407.21783}, 
}

@inproceedings{hebert2018multicalibration,
  title={Multicalibration: Calibration for the (computationally-identifiable) masses},
  author={H{\'e}bert-Johnson, Ursula and Kim, Michael and Reingold, Omer and Rothblum, Guy},
  booktitle={International Conference on Machine Learning},
  pages={1939--1948},
  year={2018},
  organization={PMLR}
}

@misc{angelopoulos2021gentle,
      title={A Gentle Introduction to Conformal Prediction and Distribution-Free Uncertainty Quantification}, 
      author={Anastasios N. Angelopoulos and Stephen Bates},
      year={2022},
      eprint={2107.07511},
      archivePrefix={arXiv},
      primaryClass={cs.LG},
      url={https://arxiv.org/abs/2107.07511}, 
}

@inproceedings{damani2025learning,
  title={Learning how hard to think: Input-adaptive allocation of lm computation},
  author={Damani, Mehul and Shenfeld, Idan and Peng, Andi and Bobu, Andreea and Andreas, Jacob},
  booktitle={International Conference on Learning Representations},
  volume={2025},
  pages={102783--102802},
  year={2025}
}

@article{detommaso2024multicalibration,
  title={Multicalibration for confidence scoring in llms},
  author={Detommaso, Gianluca and Bertran, Martin and Fogliato, Riccardo and Roth, Aaron},
  journal={arXiv preprint arXiv:2404.04689},
  year={2024}
}

@article{cherian2024large,
  title={Large language model validity via enhanced conformal prediction methods},
  author={Cherian, John J and Gibbs, Isaac and Cand{\`e}s, Emmanuel J},
  journal={Advances in Neural Information Processing Systems},
  volume={37},
  pages={114812--114842},
  year={2024}
}

@article{fang2022online,
  title={Online mirror descent and dual averaging: keeping pace in the dynamic case},
  author={Fang, Huang and Harvey, Nicholas JA and Portella, Victor S and Friedlander, Michael P},
  journal={Journal of Machine Learning Research},
  volume={23},
  number={121},
  pages={1--38},
  year={2022}
}

@article{hazan2016introduction,
  title={Introduction to online convex optimization},
  author={Hazan, Elad},
  journal={Foundations and Trends in Optimization},
  volume={2},
  number={3-4},
  pages={157--325},
  year={2016},
  publisher={Emerald Publishing Limited}
}

@inproceedings{lightman2024let,
  title={Let's verify step by step},
  author={Lightman, Hunter and Kosaraju, Vineet and Burda, Yuri and Edwards, Harrison and Baker, Bowen and Lee, Teddy and Leike, Jan and Schulman, John and Sutskever, Ilya and Cobbe, Karl},
  booktitle={International Conference on Learning Representations},
  volume={2024},
  pages={39578--39601},
  year={2024}
}

@article{farquhar2024detecting,
  title={Detecting hallucinations in large language models using semantic entropy},
  author={Farquhar, Sebastian and Kossen, Jannik and Kuhn, Lorenz and Gal, Yarin},
  journal={Nature},
  volume={630},
  number={8017},
  pages={625--630},
  year={2024},
  publisher={Nature Publishing Group UK London}
}

@article{ong2024routellm,
  title={Routellm: Learning to route llms with preference data},
  author={Ong, Isaac and Almahairi, Amjad and Wu, Vincent and Chiang, Wei-Lin and Wu, Tianhao and Gonzalez, Joseph E and Kadous, M Waleed and Stoica, Ion},
  journal={arXiv preprint arXiv:2406.18665},
  year={2024}
}

@article{cobbe2021training,
  title={Training verifiers to solve math word problems},
  author={Cobbe, Karl and Kosaraju, Vineet and Bavarian, Mohammad and Chen, Mark and Jun, Heewoo and Kaiser, Lukasz and Plappert, Matthias and Tworek, Jerry and Hilton, Jacob and Nakano, Reiichiro and others},
  journal={arXiv preprint arXiv:2110.14168},
  year={2021}
}

@article{chen2023frugalgpt,
  title={Frugalgpt: How to use large language models while reducing cost and improving performance},
  author={Chen, Lingjiao and Zaharia, Matei and Zou, James},
  journal={arXiv preprint arXiv:2305.05176},
  year={2023}
}

@article{nemirovski2004prox,
  title={Prox-method with rate of convergence O (1/t) for variational inequalities with Lipschitz continuous monotone operators and smooth convex-concave saddle point problems},
  author={Nemirovski, Arkadi},
  journal={SIAM Journal on Optimization},
  volume={15},
  number={1},
  pages={229--251},
  year={2004},
  publisher={SIAM}
}

@inproceedings{he2020momentum,
  title={Momentum contrast for unsupervised visual representation learning},
  author={He, Kaiming and Fan, Haoqi and Wu, Yuxin and Xie, Saining and Girshick, Ross},
  booktitle={Proceedings of the IEEE/CVF conference on computer vision and pattern recognition},
  pages={9729--9738},
  year={2020}
}

@article{kivinen1997exponentiated,
  title={Exponentiated gradient versus gradient descent for linear predictors},
  author={Kivinen, Jyrki and Warmuth, Manfred K},
  journal={Information and computation},
  volume={132},
  number={1},
  pages={1--63},
  year={1997},
  publisher={Elsevier}
}

@inproceedings{ni2023lever,
  title={Lever: Learning to verify language-to-code generation with execution},
  author={Ni, Ansong and Iyer, Srini and Radev, Dragomir and Stoyanov, Ves and Yih, Wen-tau and Wang, Sida I and Lin, Xi Victoria},
  booktitle={Proceedings of the 40th International Conference on Machine Learning (ICML'23)},
  year={2023}
}

@article{zhang2024fair,
  title={Fair risk control: A generalized framework for calibrating multi-group fairness risks},
  author={Zhang, Lujing and Roth, Aaron and Zhang, Linjun},
  journal={arXiv preprint arXiv:2405.02225},
  year={2024}
}

\appendix
\section*{Supplementary Material}
\section{Proof of the Allocation-Distortion Proposition}
\label{sec:proofs}

\begin{proof}
Let $z=s/\hat{c}$ and let the population be partitioned into $K$ strata with weights $w_k$. Write $F_k$ for the within-stratum CDF of $z$, $F=\sum_k w_kF_k$ for the mixture CDF, and $g_k(z)=p f_k(z\mid y{=}1)$ for the error-weighted density under the shared prevalence assumption in (A3).

\textbf{Step 1: Stratum-wise decomposition.}
Using the law of total expectation,
\[
V(\pi) = \sum_{k=1}^K w_k \cdot \mathbb{E}[y \cdot \mathbf{1}(z > \tau) \mid k].
\]
Thus the allocation-value gap decomposes into per-stratum contributions.

\textbf{Step 2: Global threshold as mixture distortion.}
The global threshold $\tau$ satisfies the mixture budget constraint:
\[
\sum_k w_k \mathbb{P}(z > \tau \mid k) = \beta.
\]
Equivalently, $\tau=Q_{1-\beta}(F)$, and the stratified policy uses $\tau_k=Q_{1-\beta}(F_k)$ within each stratum. By the local inverse-Lipschitz condition in (A2), the quantile perturbation satisfies
\begin{equation}
|\tau-\tau_k| \leq \frac{1}{\ell}\sup_z|F_k(z)-F(z)|.
\label{eq:quantile-perturb}
\end{equation}

\textbf{Step 3: CDF deviation under heterogeneous discriminability.}
Under the common-prevalence, common-scale Gaussian location model in (A3), the stratum CDFs differ through $\delta_k$. Since $\|\Phi'\|_\infty=1/\sqrt{2\pi}$,
\[
\sup_z|F_k(z)-F_j(z)|
\leq \frac{p}{\sqrt{2\pi}}|\delta_k-\delta_j|.
\]
Because $F=\sum_jw_jF_j$, combining this inequality with Eq.~(\ref{eq:quantile-perturb}) gives
\[
|\tau-\tau_k|
\leq \frac{p}{\ell\sqrt{2\pi}}\sum_jw_j|\delta_k-\delta_j|.
\]

\textbf{Step 4: Composition into allocation-value gap.}
Assumption (A1) implies that changing the threshold within stratum $k$ changes discovered error mass by at most
\[
\left|\int_{\min(\tau,\tau_k)}^{\max(\tau,\tau_k)} g_k(z)\,dz\right|
\leq M|\tau-\tau_k|.
\]
Therefore,
\begin{align*}
|V(\pi_{\mathrm{strat}})-V(\pi_{\mathrm{global}})|
&\leq \sum_kw_kM|\tau-\tau_k|\\
&\leq \frac{Mp}{\ell\sqrt{2\pi}}
\sum_{k,j}w_kw_j|\delta_k-\delta_j|\\
&\leq \frac{2Mp}{\ell\sqrt{2\pi}}
\sum_kw_k|\delta_k-\bar{\delta}|,
\end{align*}
where the final step uses
$|\delta_k-\delta_j|\leq|\delta_k-\bar{\delta}|+|\delta_j-\bar{\delta}|$.
This proves the stated bound. Under (A3), equal $\delta_k$ imply identical $F_k$, so $\tau_k=\tau$ and the gap vanishes. The result is an upper bound on allocation-value divergence, not a sign guarantee for empirical hit-rate gains.
\end{proof}

\section{Full Per-Budget Results on MBPP Qwen3-8B}
\label{appendix:qwen3}

Table~\ref{tab:cst_main} reports mean per-budget results for MBPP Qwen3-8B across 10 seeds, expanding the main-paper summary. CST achieves the best non-Oracle mean hit rate at every budget. The homogeneous-cost version appears in the homogeneous-cost control section below.

\begin{table}[t]
\centering
\footnotesize
\setlength{\tabcolsep}{2.8pt}
\caption{Mean hit rate across 10 seeds, MBPP Qwen3-8B (heterogeneous costs, $N=500$).}
\label{tab:cst_main}
\begin{tabular}{lccccccc}
\toprule
$\beta$ & Rand & Thr & CST & C+M & MP-A & MP-S & Oracle \\
\midrule
5\%  & .549 & .616 & \textbf{.679} & .666 & .601 & .642 & 1.00 \\
10\% & .551 & .636 & \textbf{.671} & .622 & .597 & .656 & 1.00 \\
15\% & .562 & .620 & \textbf{.651} & .651 & .622 & .621 & 1.00 \\
20\% & .546 & .610 & \textbf{.661} & .647 & .619 & .634 & 1.00 \\
25\% & .548 & .610 & \textbf{.662} & .651 & .621 & .616 & 1.00 \\
30\% & .550 & .605 & \textbf{.661} & .638 & .616 & .599 & 1.00 \\
40\% & .564 & .591 & \textbf{.645} & .625 & .586 & .608 & 1.00 \\
50\% & .562 & .588 & \textbf{.631} & .615 & .569 & .600 & 1.00 \\
\bottomrule
\end{tabular}
\end{table}

\section{LLaMA3-8B Results on MBPP}
\label{appendix:llama3}

Table~\ref{tab:llama3_main} reports LLaMA3-8B results on MBPP (heterogeneous costs). Despite exhibiting the opposite heterogeneity pattern to Qwen3 (signal rises with cost rather than collapsing), CST gains remain strong, consistent with the diagnosis that cross-stratum dispersion---rather than its direction---creates consequential global-allocation distortion.

\begin{table}[t]
\centering
\small
\caption{Mean LLaMA3-8B hit rate across 10 seeds on MBPP (heterogeneous costs, $N=500$).}
\label{tab:llama3_main}
\setlength{\tabcolsep}{2.8pt}
\begin{tabular}{lccccccc}
\toprule
$\beta$ & Rand & Thr & CST & C+M & MP-A & MP-S & Oracle \\
\midrule
5\%  & .624 & .828 & \textbf{.898} & .882 & .749 & .875 & 1.000 \\
10\% & .630 & .712 & \textbf{.882} & .844 & .689 & .837 & 1.000 \\
15\% & .629 & .692 & \textbf{.855} & .822 & .698 & .839 & 1.000 \\
20\% & .638 & .686 & \textbf{.837} & .795 & .693 & .794 & 1.000 \\
25\% & .643 & .698 & \textbf{.821} & .781 & .702 & .784 & .989 \\
30\% & .642 & .693 & \textbf{.803} & .755 & .692 & .763 & .995 \\
40\% & .650 & .683 & \textbf{.760} & .716 & .677 & .735 & .996 \\
50\% & .649 & .675 & \textbf{.728} & .693 & .673 & .704 & 1.000 \\
\bottomrule
\end{tabular}
\end{table}

\section{MATH Full Per-Budget Results}
\label{appendix:math}

Tables~\ref{tab:math_qwen3} and \ref{tab:math_llama3} report full per-budget results on MATH ($N=5000$). Qwen3-MATH shows Threshold as the dominant policy at $\beta{\le}30\%$, with CST gaining only at $\beta{=}50\%$. LLaMA3-MATH transitions at $\beta{\ge}30\%$ due to stronger cost-error correlation ($r{=}0.309$ vs.\ $0.193$). At $\beta{=}5\%$, MP-Adapt outperforms CST on Qwen3 but not LLaMA3, reflecting the cross-task boundary condition discussed in the main paper.

\begin{table}[t]
\centering
\small
\setlength{\tabcolsep}{2.8pt}
\caption{Mean Qwen3-8B hit rate across 10 seeds on MATH (heterogeneous costs, $N=5000$, pass@1=0.753).}
\label{tab:math_qwen3}
\begin{tabular}{lccccccc}
\toprule
$\beta$ & Rand & Thr & CST & C+M & MP-A & MP-S & Oracle \\
\midrule
5\%  & .250 & \textbf{.556} & .329 & .321 & .395 & .306 & 1.000 \\
10\% & .242 & \textbf{.524} & .322 & .308 & .356 & .288 & 1.000 \\
15\% & .244 & \textbf{.443} & .314 & .308 & .330 & .290 & 1.000 \\
20\% & .246 & \textbf{.384} & .308 & .303 & .310 & .285 & .997 \\
25\% & .247 & \textbf{.348} & .303 & .301 & .294 & .271 & .997 \\
30\% & .248 & \textbf{.320} & .301 & .296 & .284 & .265 & .938 \\
40\% & .247 & .285 & \textbf{.294} & .284 & .263 & .258 & .650 \\
50\% & .248 & .264 & \textbf{.283} & .272 & .250 & .252 & .526 \\
\bottomrule
\end{tabular}
\end{table}

\begin{table}[t]
\centering
\small
\setlength{\tabcolsep}{2.8pt}
\caption{Mean LLaMA3-8B hit rate across 10 seeds on MATH (heterogeneous costs, $N=5000$, pass@1=0.387).}
\label{tab:math_llama3}
\begin{tabular}{lccccccc}
\toprule
$\beta$ & Rand & Thr & CST & C+M & MP-A & MP-S & Oracle \\
\midrule
5\%  & .611 & \textbf{.870} & .828 & .771 & .777 & .759 & 1.000 \\
10\% & .605 & \textbf{.848} & .799 & .744 & .734 & .730 & .999 \\
15\% & .603 & \textbf{.810} & .776 & .727 & .698 & .710 & .999 \\
20\% & .608 & \textbf{.774} & .758 & .718 & .681 & .699 & 1.000 \\
25\% & .611 & .743 & \textbf{.745} & .707 & .667 & .684 & .999 \\
30\% & .615 & .719 & \textbf{.734} & .698 & .653 & .671 & .999 \\
40\% & .615 & .682 & \textbf{.715} & .690 & .629 & .650 & 1.000 \\
50\% & .616 & .655 & \textbf{.693} & .668 & .615 & .632 & 1.000 \\
\bottomrule
\end{tabular}
\end{table}

\section{GPT-4o-mini Results on MBPP}
\label{appendix:gpt4omini}

Table~\ref{tab:gpt4omini_main} reports GPT-4o-mini results on MBPP (heterogeneous costs). Its CST gains coincide with extreme Q4 signal collapse ($\rho_{\mathrm{Q4}}{=}0.059$ n.s. vs.\ Qwen3's 0.157 n.s.), supporting the use of cross-stratum heterogeneity---rather than global correlation alone---as the relevant diagnostic.

\begin{table}[t]
\centering
\small
\caption{Mean GPT-4o-mini hit rate across 10 seeds on MBPP (heterogeneous costs, $N=500$).}
\label{tab:gpt4omini_main}
\setlength{\tabcolsep}{2.8pt}
\begin{tabular}{lccccccc}
\toprule
$\beta$ & Rand & Thr & CST & C+M & MP-A & MP-S & Oracle \\
\midrule
5\%  & .535 & .426 & \textbf{.630} & .604 & .455 & .595 & .977 \\
10\% & .513 & .511 & \textbf{.592} & .592 & .524 & .552 & .988 \\
15\% & .526 & .535 & .603 & \textbf{.608} & .549 & .560 & .991 \\
20\% & .527 & .545 & .610 & \textbf{.610} & .553 & .579 & .993 \\
25\% & .533 & .556 & \textbf{.610} & .607 & .557 & .582 & .994 \\
30\% & .533 & .547 & .614 & \textbf{.620} & .563 & .578 & .995 \\
40\% & .540 & .562 & .612 & \textbf{.624} & .578 & .567 & .991 \\
50\% & .539 & .578 & \textbf{.612} & .601 & .580 & .576 & .996 \\
\bottomrule
\end{tabular}
\end{table}

\section{Homogeneous Cost Control Experiment}
\label{appendix:homo}

Under homogeneous costs ($c_t \equiv 0.5$), CST degenerates exactly to Threshold (Table~\ref{tab:homo_hr}), confirming gains derive from heteroskedasticity correction. MP-Adapt underperforms Threshold (avg.\ ${-}7.5$pp across budgets), indicating that cost heterogeneity alone does not explain the underperformance of globally shared online adaptation. The primary cause is weak global signal strength ($\rho{\approx}0.178$) present under both cost settings.

\begin{table}[t]
\centering
\small
\caption{Hit rate under homogeneous costs ($c_t \equiv 0.5$, MBPP, Qwen3-8B).}
\label{tab:homo_hr}
\setlength{\tabcolsep}{2.8pt}
\begin{tabular}{lccccccc}
\toprule
$\beta$ & Rand & Thr & CST & C+M & MP-A & MP-S & Oracle \\
\midrule
5\%  & .480 & .680 & .680 & \textbf{.800} & .560 & .600 & 1.000 \\
10\% & .440 & .740 & .740 & \textbf{.660} & .580 & .580 & 1.000 \\
20\% & .480 & \textbf{.650} & \textbf{.650} & \textbf{.650} & .590 & .610 & 1.000 \\
30\% & .553 & .653 & .653 & .647 & .673 & \textbf{.693} & 1.000 \\
50\% & .576 & .669 & .669 & \textbf{.658} & .612 & .636 & 1.000 \\
\bottomrule
\end{tabular}
\end{table}

\section{Memory Augmentation Ablation}
\label{appendix:memory}

Memory-augmented CST retrieves verified samples within the same stratum as a fourth signal $h_t^{(4)}$. At $\beta{=}20\%$, optimal weight $w{\approx}0.1$--$0.3$ yields marginal gains ($+1.0$--$+3.0$pp over CST 0.644), while $w{\ge}0.5$ causes collapse (0.515 vs.\ CST 0.644). The memory signal correlates weakly with errors ($r{=}0.059$, $p{=}0.47$), confirming that memory amplifies but cannot create signal where none exists. GPT-4o-mini shows similarly weak memory gains ($+1.4$pp max).

\section{Per-Stratum Split-Conformal Baseline}
\label{appendix:conformal}

Table~\ref{tab:strat_criteria_supp} reports the negative-control partition comparison used in the main-text diagnosis. Only the ex-ante cost partition improves over global Threshold, while random partitions and noisy signal-strength bins degrade performance.

\begin{table}[t]
\centering
\small
\caption{Stratification criterion comparison ($\beta=20\%$).}
\label{tab:strat_criteria_supp}
\begin{tabular}{lcc}
\toprule
Criterion & Hit Rate & $\Delta$ vs.\ Global \\
\midrule
None (global) & 0.604 & --- \\
Random partition & 0.590 & $-$0.014 \\
By signal strength & 0.471 & $-$0.133 \\
\textbf{By ex-ante cost (CST)} & \textbf{0.644} & \textbf{$+$0.040} \\
Oracle partition & 1.000 & $+$0.396 \\
\bottomrule
\end{tabular}
\end{table}

Table~\ref{tab:conformal} reports per-stratum split-conformal quantile calibration (50\% calibration split, fixed thresholds). This group-conditional conformal baseline partitions by cost like CST but uses one-shot calibrated thresholds. Conformal stratification improves over global Threshold universally, confirming stratification is the primary gain driver. CST's online quantile adaptation adds further gains in most regimes; conformal exceeds CST only at GPT-4o-mini $\beta{=}10\%$ ($+9.5$pp, where fixed calibration benefits from larger effective calibration samples) and LLaMA3-MATH $\beta{=}20\%$ ($+1.5$pp).

\begin{table}[t]
\centering
\footnotesize
\setlength{\tabcolsep}{2.2pt}
\caption{Per-stratum split-conformal hit rate and $\Delta$ vs.\ CST (pp, positive = exceeds CST).}
\label{tab:conformal}
\begin{tabular}{lccc}
\toprule
$\beta$ & Qwen3 & LLaMA3 & GPT-4o-mini \\
\midrule
\multicolumn{4}{c}{\textit{MBPP}} \\
\midrule
10\% & .667\,($-3.1$) & .852\,($-3.1$) & .648\,($+9.5$) \\
20\% & .626\,($-1.7$) & .865\,($+2.0$) & .644\,($+2.4$) \\
30\% & .650\,($-1.9$) & .797\,($-2.0$) & .604\,($-3.2$) \\
50\% & .640\,($-1.5$) & .749\,($+0.7$) & .616\,($+0.1$) \\
\midrule
\multicolumn{4}{c}{\textit{MATH}} \\
\midrule
10\% & .315\,($-0.8$) & .822\,($+0.5$) & --- \\
20\% & .309\,($-1.0$) & .769\,($+1.5$) & --- \\
30\% & .305\,($+0.8$) & .741\,($+0.3$) & --- \\
50\% & .289\,($+0.9$) & .704\,($+0.5$) & --- \\
\bottomrule
\end{tabular}
\end{table}

\section{Sensitivity to Number of Strata ($K$)}
\label{appendix:ksweep}

Table~\ref{tab:k_sweep} reports CST hit rate vs.\ $K$ on MATH ($K{=}1$ is global Threshold). Qwen3-MATH degrades monotonically with $K$ under weak heterogeneity ($r{=}0.193$); LLaMA3-MATH ($r{=}0.309$) peaks at $K{=}3$--$4$. Over-partitioning ($K{\ge}6$) consistently degrades performance, confirming that within-stratum sample sizes must be sufficient for reliable quantile estimation, as predicted by the allocation-distortion analysis.

\begin{table}[t]
\centering
\small
\setlength{\tabcolsep}{2.8pt}
\caption{CST hit rate vs.\ strata count $K$ (MATH). $K{=}1$ is global Threshold.}
\label{tab:k_sweep}
\begin{tabular}{lccccc|ccccc}
\toprule
 & \multicolumn{5}{c}{Qwen3-8B} & \multicolumn{5}{c}{LLaMA3-8B} \\
\cmidrule(lr){2-6}\cmidrule(lr){7-11}
$K$ & 5\% & 10\% & 20\% & 30\% & 50\% & 5\% & 10\% & 20\% & 30\% & 50\% \\
\midrule
1 & .569 & .529 & .384 & .316 & .266 & .869 & .839 & .769 & .718 & .658 \\
2 & .434 & .427 & .376 & .326 & .286 & .836 & .823 & .789 & .756 & .691 \\
3 & .380 & .373 & .336 & .327 & .284 & .856 & .794 & .769 & .744 & .694 \\
4 & .341 & .323 & .317 & .298 & .279 & .823 & .789 & .745 & .736 & .695 \\
6 & .293 & .306 & .285 & .277 & .269 & .822 & .788 & .731 & .716 & .690 \\
8 & .295 & .260 & .257 & .259 & .253 & .785 & .751 & .707 & .699 & .670 \\
\bottomrule
\end{tabular}
\end{table}

\section{Full HGA Gate Results}
\label{appendix:hga}

Table~\ref{tab:hga_full} reports Heterogeneity-Gated Allocation (HGA) across budgets in the calibration reanalysis. HGA uses the same 50-sample score initialization as Threshold and CST, then uses a separate 50\% diagnostic split to decide whether the CST gate should open. The pattern matches the intended role of the gate: HGA stays close to CST on MBPP, where cost-aligned stratification is useful, and reverts to Threshold on Qwen3-MATH, where fixed stratification is harmful at low and medium budgets.

Table~\ref{tab:hga_gate} summarizes the gate decisions at $\beta{=}20\%$. These decisions are made before deployment on the evaluation stream; HGA does not choose the better test-set policy after observing test outcomes.

\begin{table}[t]
\centering
\small
\setlength{\tabcolsep}{3pt}
\caption{HGA gate decisions at $\beta{=}20\%$ in the calibration reanalysis.}
\label{tab:hga_gate}
\begin{tabular}{lccc}
\toprule
Setting & Gate & Policy & Outcome \\
\midrule
MBPP-Qwen3 & open & CST & retains gain \\
MBPP-LLaMA3 & open & CST & retains gain \\
MBPP-GPT-4o-mini & open & CST & partial gain \\
MATH-Qwen3 & closed & Threshold & avoids loss \\
MATH-LLaMA3 & closed & Threshold & near Threshold \\
\bottomrule
\end{tabular}
\end{table}

\begin{table}[t]
\centering
\small
\setlength{\tabcolsep}{3pt}
\caption{Full HGA calibration-reanalysis hit rate results across budgets (heterogeneous costs, 10 seeds). Bold indicates the best among Threshold, CST, and HGA.}
\label{tab:hga_full}
\begin{tabular}{llccc}
\toprule
Setting & $\beta$ & Threshold & CST & HGA \\
\midrule
MBPP-Qwen3 & 10\% & .616 & \textbf{.681} & .677 \\
MBPP-Qwen3 & 20\% & .624 & \textbf{.654} & .653 \\
MBPP-Qwen3 & 30\% & .606 & \textbf{.656} & \textbf{.656} \\
MBPP-Qwen3 & 50\% & .579 & \textbf{.631} & \textbf{.631} \\
MBPP-LLaMA3 & 10\% & .689 & \textbf{.847} & .826 \\
MBPP-LLaMA3 & 20\% & .677 & \textbf{.813} & .804 \\
MBPP-LLaMA3 & 30\% & .686 & \textbf{.787} & .770 \\
MBPP-LLaMA3 & 50\% & .674 & \textbf{.725} & \textbf{.725} \\
MBPP-GPT-4o-mini & 10\% & .495 & \textbf{.571} & .568 \\
MBPP-GPT-4o-mini & 20\% & .533 & \textbf{.610} & .591 \\
MBPP-GPT-4o-mini & 30\% & .553 & \textbf{.614} & .603 \\
MBPP-GPT-4o-mini & 50\% & .572 & \textbf{.605} & .598 \\
MATH-Qwen3 & 10\% & \textbf{.525} & .330 & \textbf{.525} \\
MATH-Qwen3 & 20\% & \textbf{.385} & .307 & \textbf{.385} \\
MATH-Qwen3 & 30\% & \textbf{.321} & .300 & .316 \\
MATH-Qwen3 & 50\% & .264 & \textbf{.284} & .280 \\
MATH-LLaMA3 & 10\% & \textbf{.844} & .798 & .835 \\
MATH-LLaMA3 & 20\% & \textbf{.775} & .756 & .774 \\
MATH-LLaMA3 & 30\% & .717 & \textbf{.734} & .722 \\
MATH-LLaMA3 & 50\% & .656 & \textbf{.695} & .668 \\
\bottomrule
\end{tabular}
\end{table}

\section{Budgeted Model Routing Reanalysis}
\label{appendix:routing}

We further reinterpret the MBPP per-instance records as a budgeted model-routing problem. The cheap model is run first, and routing to GPT-4o-mini is valuable only when the cheap model fails and GPT-4o-mini succeeds. Policies use the cheap model's uncertainty signals and an ex-ante prompt-cost proxy; the metric is routing hit rate among selected examples. This is a reanalysis of existing records, not a new model run.

Table~\ref{tab:routing_reanalysis} shows a conditional pattern. For Qwen3, routing benefits are sparse (9.6\% of examples), and global Threshold is best at $\beta{=}20\%$. For LLaMA3, routing benefits are denser (15.8\%), and stratified allocation improves over Threshold: MP-Strat reaches .183 and CST .180 vs.\ .173. Thus the routing setting supports the paper's diagnostic interpretation rather than a universal claim that CST must dominate: stratification helps when the benefit signal has exploitable conditional structure, and simple global thresholding can remain competitive when the benefit label is too sparse.

\begin{table}[t]
\centering
\small
\setlength{\tabcolsep}{2.2pt}
\caption{Budgeted model-routing reanalysis on MBPP at $\beta{=}20\%$. Hit rate measures the fraction of routed examples for which the cheap model fails and GPT-4o-mini succeeds. Bold indicates the best non-oracle policy in each row.}
\label{tab:routing_reanalysis}
\begin{tabular}{llccccc}
\toprule
Route & $\beta$ & Rand. & Thr. & MP-A & MP-S & CST \\
\midrule
Qwen3$\rightarrow$GPT-4o-mini & 20\% & .096 & \textbf{.109} & .098 & .100 & .103 \\
LLaMA3$\rightarrow$GPT-4o-mini & 20\% & .153 & .173 & .149 & \textbf{.183} & .180 \\
\bottomrule
\end{tabular}
\end{table}

\section{Visual Illustration of Global vs.\ Stratum-Adaptive Thresholding}
\label{appendix:fig1}

\begin{figure*}[t]
\centering
\includegraphics[width=\textwidth]{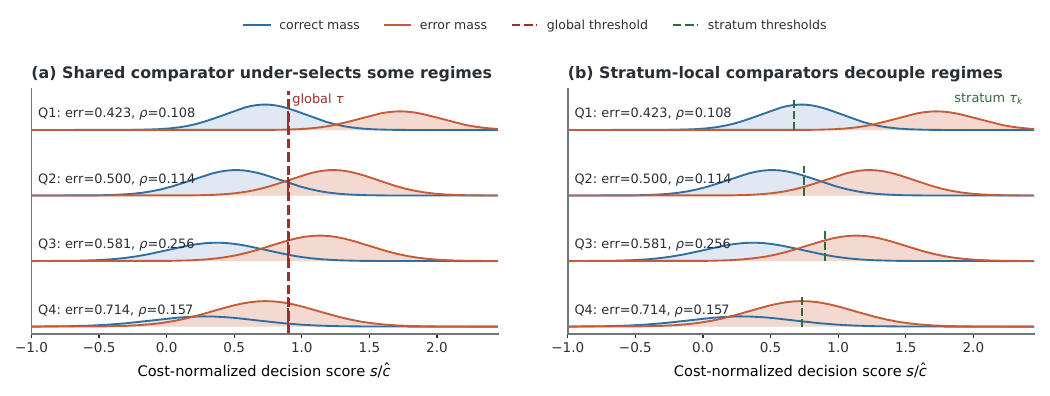}
\caption{Global thresholding under heteroskedastic costs. Synthetic score distributions parameterized by empirical $\rho_k$ and error rates (MBPP Qwen3-8B, $\beta{=}20\%$). Decision scores incorporate the cost denominator: higher cost deflates scores in Q4, inflates in Q1. \textbf{Left:} a single global threshold $\tau$ under-selects high-cost error-dense strata (Q4, highest error) and strong-signal strata (Q3), while over-selecting low-cost strata (Q1, Q2). \textbf{Right:} stratum-adaptive thresholds $\tau_k$ (CST) computed via per-stratum quantiles correct the misallocation.}
\label{fig:cst_vs_threshold_appendix}
\end{figure*}
\end{document}